\documentclass[sigconf]{acmart}

\AtBeginDocument{%
  \providecommand\BibTeX{{%
    \normalfont B\kern-0.5em{\scshape i\kern-0.25em b}\kern-0.8em\TeX}}}

\setcopyright{acmcopyright}
\copyrightyear{2018}
\acmYear{2018}
\acmDOI{10.1145/1122445.1122456}

\acmConference[The Web Conference '22]{The Web Conference}{April 25--29, 2022}{Lyon, France}
\acmBooktitle{The Web Conference, April 25--29, 2022, Lyon, France}
\acmPrice{15.00}
\acmISBN{978-1-4503-XXXX-X/18/06}

\usepackage{amsfonts}
\usepackage[ruled,vlined,shortend, linesnumbered]{algorithm2e} 
\usepackage{algpseudocode}
\usepackage{subcaption}
\usepackage{soul}
\usepackage{graphicx}
\usepackage{hyperref}
\usepackage{xcolor}
\usepackage{color, colortbl}
\usepackage{helvet}
\usepackage{courier}
\usepackage{multirow}
\usepackage{booktabs}
\usepackage{placeins}
\newsavebox{\measurebox}
\usepackage{pgfplots}
\pgfplotsset{compat=1.7}
\usepackage{tikz}
\usetikzlibrary{automata,positioning}

\newtheorem{proposition}{Proposition}

\renewcommand{\vec}[1]{\mathbf{#1}}

\usepackage{tikz}
\usetikzlibrary{automata,positioning}
\newcommand{\method}{\textsc{MemStream}}

\usepackage{Definitions}
\newcommand{\hx}{{\hat x}}

\newcommand{\bbeta}{{\bar \beta}}
\newcommand{\bd}{{\bar d}}

\copyrightyear{2022}
\acmYear{2022}
\setcopyright{rightsretained}
\acmConference[WWW '22]{Proceedings of the ACM Web Conference 2022}{April 25--29, 2022}{Virtual Event, Lyon, France}
\acmBooktitle{Proceedings of the ACM Web Conference 2022 (WWW '22), April 25--29, 2022, Virtual Event, Lyon, France}\acmDOI{10.1145/3485447.3512221}
\acmISBN{978-1-4503-9096-5/22/04}

\begin{document}

\title{\method: Memory-Based Streaming Anomaly Detection}


\author{Siddharth Bhatia}

\affiliation{%
  \institution{National University of Singapore}
  \country{Singapore}
}
\email{siddharth@comp.nus.edu.sg}

\author{Arjit Jain}
\affiliation{%
  \institution{IIT Bombay}
  \country{India}
}
\email{arjit@cse.iitb.ac.in}

\author{Shivin Srivastava}
\affiliation{%
  \institution{National University of Singapore}
  \country{Singapore}
}
\email{shivin@comp.nus.edu.sg}

\author{Kenji Kawaguchi}
\affiliation{%
  \institution{Harvard University}
  \country{United States}
}
\email{kkawaguchi@fas.harvard.edu}

\author{Bryan Hooi}
\affiliation{%
  \institution{National University of Singapore}
  \country{Singapore}
}
\email{bhooi@comp.nus.edu.sg}


\begin{abstract}
Given a stream of entries over time in a multi-dimensional data setting where concept drift is present, how can we detect anomalous activities? Most of the existing unsupervised anomaly detection approaches seek to detect anomalous events in an offline fashion and require a large amount of data for training. This is not practical in real-life scenarios where we receive the data in a streaming manner and do not know the size of the stream beforehand. Thus, we need a data-efficient method that can detect and adapt to changing data trends, or \textit{concept drift}, in an online manner. In this work, we propose \textbf{\method}, a streaming anomaly detection framework, allowing us to detect unusual events as they occur while being resilient to concept drift. We leverage the power of a denoising autoencoder to learn representations and a memory module to learn the dynamically changing trend in data without the need for labels. We prove the optimum memory size required for effective drift handling. Furthermore, \method\ makes use of two architecture design choices to be robust to memory poisoning. Experimental results show the effectiveness of our approach compared to state-of-the-art streaming baselines using $2$ synthetic datasets and $11$ real-world datasets.
\end{abstract}



\maketitle

\section{Introduction}
Anomaly detection is a fundamental and well-studied problem in many areas, such as cybersecurity \citep{tan2011fast,bhatia2020midas}, video surveillance \citep{Mahadevan2010AnomalyDI,ravanbakhsh2019training}, financial fraud \citep{srivastava2008credit} and healthcare \citep{schlegl2017unsupervised}. Traditional classifiers trained in a supervised learning setting do not work well in anomaly detection because of the cold-start problem, i.e., the amount of anomalous data is usually not sufficient to train the model. Therefore, anomaly detectors are trained in an unsupervised setting where the normal data distribution is learned and instances that appear unlikely under this distribution are identified as anomalous.

Developing effective methods for handling \textit{multi-aspect data} (i.e. data having multiple features or dimensions) still remains a challenge. This is especially true in an unsupervised setting, where traditional anomaly detection algorithms, such as One-Class SVM, tend to perform poorly because of the curse of dimensionality. Deep architectures such as Autoencoders \citep{hinton1994autoencoders}, because of their ability to learn multiple levels of representation, are able to achieve better performance compared to their shallow counterparts \citep{bengio2009learning}. For anomaly detection, existing deep learning based techniques include deep belief networks \citep{erfani2016high}, variational autoencoders \citep{an2015variational,xu2018unsupervised}, adversarial autoencoders \citep{beggel2019robust,lim2018doping,zhai2016deep}, and deep one-class networks \citep{chalapathy2018anomaly,pmlr-v80-ruff18a}.

The problem of anomaly detection becomes even more challenging when the data arrives in a streaming/online manner and we want to detect anomalies in real-time. For example, intrusions in cybersecurity need to be detected as soon as they arrive to minimize the harm caused. Moreover, in streaming data, there can be a drift in the distribution over time which the existing approaches \citep{Hariri2021ExtendedIF,Bhatia2021MSTREAM,Manzoor2018xStreamOD,Na2018DILOFEA,Mirsky2018KitsuneAE,guha2016robust} are unable to fully handle.

To handle concept drift in a streaming setting, our approach uses an explicit memory module. For anomaly detection, this memory can be used to store the trends of normal data that act as a baseline with which to judge incoming records. A read-only memory, in a drifting setting, is of limited use and thus should be accompanied by an appropriate memory update strategy. The records arrive over time; thus, older records in the memory might no longer be relevant to the current trends suggesting a First-In-First-Out memory replacement strategy. The introduction of memory, with an appropriate update strategy, seems to tackle some of the issues in streaming anomaly detection with concept drift. However, the system described so far does not provide a fail-safe for when an anomalous sample enters the memory and is thus susceptible to memory poisoning.

We, therefore, propose \method, which uses a denoising autoencoder \citep{denoisingae} to extract features, and a memory module to learn the dynamically changing trend, thereby avoiding the over-generalization of autoencoders (i.e. the problem of autoencoders reconstructing anomalous samples well). Our streaming framework is resilient to concept drift and we prove a theoretical bound on the size of memory for effective drift handling. Moreover, we allow quick retraining when the arriving stream becomes sufficiently different from the training data.

We also discuss two architectural design choices to make \method\ robust to memory poisoning. The first modification prevents anomalous elements from entering the memory, and the second modification deals with how the memory can be self-corrected and recovered even if it harbors anomalous elements. Finally, we discuss the effectiveness of \method\ compared to state-of-the-art streaming baselines.

In summary, the main contributions of our paper are:
\begin{enumerate}
    \item {\bf Streaming Anomaly Detection:} We propose a novel streaming approach using a denoising autoencoder and a memory module, for detecting anomalies. \method\ is resilient to concept drift and allows quick retraining.
    \item {\bf Theoretical Guarantees:} In Proposition \ref{prop:1}, we discuss the optimum memory size for effective concept drift handling. In Proposition \ref{prop:2}, we discuss the motivation behind our architecture design.
    \item {\bf Robustness to Memory Poisoning:} \method\ prevents anomalies from entering the memory and can self-correct and recover from bad memory states.
    \item {\bf Effectiveness:} Our experimental results show that \method\ convincingly outperforms $11$ state-of-the-art baselines using $2$ synthetic datasets (that we release as open-source) and $11$ popular real-world datasets.
\end{enumerate}

{\bf Reproducibility}: Our code and datasets are available on \href{https://github.com/Stream-AD/MemStream}{https://github.com/Stream-AD/MemStream}.


\section{Related Work}
\citep{chandola2009anomaly} surveys traditional anomaly detection methods including reconstruction-based approaches \citep{jolliffe1986principal,gunter2007fast,lovric2011international,candes2011robust,kim2009observe,lu2013abnormal,zhao2017spatio}, clustering-based \citep{zimek2012survey,xiong2011group,barnett1984discordancy,aggarwal2015outlier,kim2012robust}, one class classification-based \citep{scholkopf2001estimating,scholkopf2000support,williams2002comparative}. Several deep learning based methods have also been proposed for anomaly detection such as GAN-based approaches \citep{yang2020memgan,Bashar2020TAnoGANTS,ngo2019,zenati2018adversarially,deecke2018image,akcay2018ganomaly,schlegl2017unsupervised}, Energy-based \citep{kumar2019maximum,zhai2016deep}, Autoencoder-based \citep{gong2019memorizing,Su2019RobustAD,zong2018deep,xu2018unsupervised,zhou2017anomaly,Slch2016VariationalIF,an2015variational}, and RNN-based \citep{saurav2018rnn_online_anomaly}; see \citep{chalapathy2019deep,pang2020deep} for extensive surveys. However, deep learning based approaches such as MemAE \cite{gong2019memorizing} do not process the data in a streaming manner and typically require a large amount of training data in an offline setting, whereas we process the data in an online manner. Additionally, we provide theoretical analysis and are robust to memory poisoning. Anomaly detection is a vast topic by itself and cannot be fully covered in this manuscript; in this section, our review mainly focuses on methods that can detect anomalies in streams containing concept drift; see \citep{lu2018learning,gupta2013outlier} for concept drift literature and \cite{Pasricha2018IdentifyingAA,Benczr2019ReinforcementLU,Chi2018HashingFA,Shao2014PrototypebasedLO,Bai2016AnOM} for different ways to detect concept drift in streams.

As for density-based approaches, Local Outlier Factor (LOF) \citep{breunig2000lof} estimates the local density at each point, then identifies anomalies as points with much lower local density than their neighbors. DILOF \citep{Na2018DILOFEA} improves upon LOF and LOF variants \citep{Salehi2016FastME,Pokrajac2007IncrementalLO} by adopting a novel density-based sampling scheme to summarize the data, without prior assumptions on the data distribution. LUNAR \citep{goodge2021lunar} is a hybrid approach combining deep learning and LOF. However, LOF-based approaches are suitable only for lower-dimensional data due to the curse of dimensionality.

Isolation Forest (IF) \citep{liu2008isolation} constructs trees by randomly selecting features and splitting them at random split points, and then defines anomalies as points that are separated from the rest of the data at low depth values. HS-Tree \citep{Tan2011FastAD} uses an ensemble of randomly constructed half-space trees with a sliding window to detect anomalies in evolving streaming data. iForestASD \citep{Ding2013AnAD} uses a sliding window frame scheme to handle abnormal data. Random Cut Forest (RCF) \citep{guha2016robust} tries to further improve upon IF  by creating multiple random cuts (trees) of data and constructing a forest of such trees to determine whether a point is anomalous or not. Recently, \citep{Hariri2021ExtendedIF} shows that splitting by only one variable at a time introduces some biases in IF which can be overcome by using hyperplane cuts instead. They propose Extended Isolation Forest (Ex. IF) \citep{Hariri2021ExtendedIF} where the split criterion is based on a threshold set on a linear combination of randomly chosen variables instead of a threshold on a single variables value at a time. However, these approaches compute an anomaly score by traversing a tree structure that is bounded by the maximum depth parameter and the size of the sliding window, therefore they do not capture long-range dependence.

Popular streaming approaches include STORM \citep{Angiulli2007DetectingDO}, which uses a sliding window to detect global distance-based outliers in data streams with respect to the current window. RS-Hash \citep{Sathe2016SubspaceOD} uses subspace grids and randomized hashing in an ensemble to detect anomalies. For each model in the ensemble, a grid is constructed using subsets of features and data, random hashing is used to record data counts in grid cells, and the anomaly score of a data point is the log of the frequency in its hashed bins. LODA \citep{Pevn2015LodaLO} generates several weak anomaly detectors by producing many random projections of the data and then computing a density estimation histogram for each projection. The outlier scores produced are the mean negative log-likelihood according to each histogram for each point. \textsc{xStream} \citep{Manzoor2018xStreamOD} detects anomalies in feature-evolving data streams through the use of a streaming random projection scheme and ensemble of half-space chains. \textsc{MStream} \citep{Bhatia2021MSTREAM} performs feature extraction and then detects group anomalies in multi-aspect streams. Kitsune \citep{Mirsky2018KitsuneAE} is an ensemble of light-weight autoencoders for real-time anomaly detection. We compare with all these methods in Section \ref{sec:exp}.

\section{Problem}

Let $\mathcal{X} = \{x_1, x_2, \cdots\}$ be records arriving in a streaming manner. Each entry $x_i = (x_{i1}, \cdots, x_{id})$ consisting of $d$ \emph{attributes} or dimensions, where each dimension can either be categorical (e.g. IP address) or real-valued (e.g. average packet length).

Our goal is to detect anomalies in streaming data. A common phenomenon in real-world data is that the nature of the stream changes over time. These changes are generally described in terms of the statistical properties of the stream, such as the mean changes across some or all features. As the definition of the ``concept" of normal behavior changes, so does the definition of an anomaly. Thus, we need a model that is able to adapt to the dynamic trend and thereby recognize anomalous records.

\section{Algorithm}

\subsection{Motivation}

\begin{table}[!ht]
\centering
\caption{Simple toy example, consisting of a stream of records over time with a trend shift at $t=6$.}
\label{tab:toy}
\begin{tabular}{@{}rrrrr@{}}
\toprule
{\bf Time} & {\bf Feature 1} & {\bf Feature 2} & {\bf Feature 3} & {\bf ...} \\ \midrule
$1$ & $8.39$ & $1.44$ & $4.16$ & $\cdots$ \\
$2$ & $6.72$ & $4.55$ & $3.49$ & $\cdots$ \\
$3$ & $3.49$ & $2.10$ & $1.56$ & $\cdots$ \\
$4$ & $4.28$ & $0.64$ & $1.22$ & $\cdots$ \\
$5$ & $5.54$ & $2.40$ & $6.55$ & $\cdots$ \\
$6$ & $183.75$ & $132.03$ & $9.86$ & $\cdots$ \\
$7$ & $146.47$ & $128.49$ & $16.52$ & $\cdots$ \\
$8$ & $197.96$ & $97.16$ & $15.05$ & $\cdots$ \\
$9$ & $192.50$ & $89.95$ & $12.46$ & $\cdots$ \\
$10$ & $158.32$ & $10.37$ & $15.76$ & $\cdots$ \\
\bottomrule
\end{tabular}
\end{table}

Consider an attacker who hacks a particular IP address and uses it to launch denial of service attacks on a server. Modern cybersecurity systems are trained to detect and block such attacks, but this is made more challenging by changes over time, e.g. in the identification of attacking machines. This is a ``concept" drift and the security system must learn to identify such changing trends to mitigate the attacks. Consider the toy example in Table \ref{tab:toy}, comprising of a multi-dimensional temporal data stream. There is a sudden distribution change and concept drift in all attributes from time $t=5$ to $t=6$.

The main challenge for the algorithm is to detect these types of patterns in a {\bf streaming} manner within a suitable timeframe. That is, the algorithm should not give an impulsive reaction to a short-lived change in the base distribution, but also should not take too long to adapt to the dynamic trend. Note that we do not want to set any limits a priori on the duration of the anomalous activity we want to detect, or the window size after which the model should be updated to account for the concept drift.



\begin{figure*}[!t]
\begin{centering}
  \includegraphics[width=\textwidth]{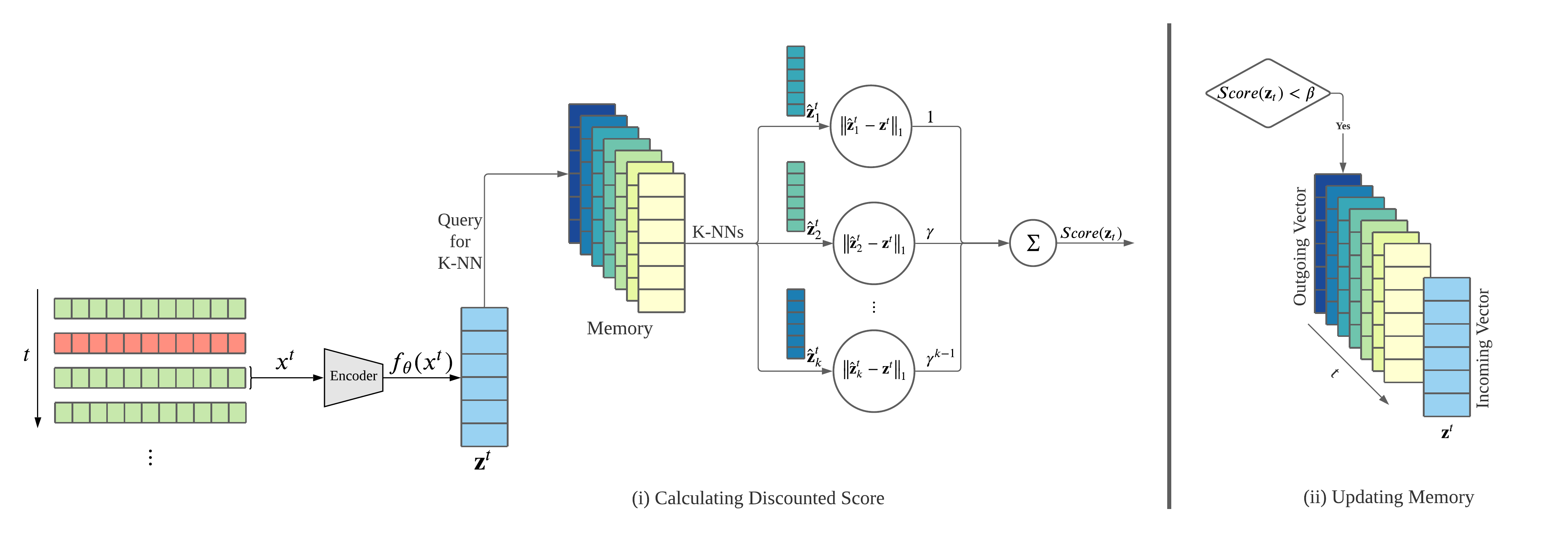}
\caption{
After an initial training of the feature extractor on a small subset of normal data, \method\ processes records in two steps: (i) It outputs anomaly scores for each record by querying the memory for $K$-nearest neighbours to the record encoding and calculating a discounted distance and (ii) It updates the memory, in a FIFO manner, if the anomaly score is within an update threshold $\beta$.}
  \label{fig:teaser}
  \end{centering}
\end{figure*}

\subsection{Overview}
As shown in Figure \ref{fig:teaser}, the proposed \method\ algorithm addresses these problems through the use of a memory augmented feature extractor that is initially trained on a small subset of normal data. The memory acts as a reserve of encodings of normal data. At a high level, the role of the feature extractor is to capture the structure of normal data. An incoming record is then scored by \emph{calculating the discounted score} based on the similarity of its encoding as evaluated against those in memory. Based on this score, if the record is deemed normal, then it is used to \emph{update the memory}. To adapt to the changing data trend, memory is required to keep track of the data drift from the original distribution. Since concept drift is generally a gradual process, the memory should maintain the temporal contiguity of records. This is achieved by following a First-In-First-Out (FIFO) memory replacement policy.

\subsection{Feature Extraction}
\label{sec:representation}
Neural Networks can learn representations using an autoencoder consisting of two parts - an encoder and a decoder \citep{Goodfellow-et-al-2016}. The encoder forms an intermediate representation of the input samples and the decoder is trained to reconstruct the input samples from their intermediate representations. Denoising autoencoders~\citep{denoisingae} partially corrupt the input data before passing it through the encoder. Intuitively, this ``forces" the network to capture the useful structure in the input distribution, pushing it to learn more robust features of the input. In our implementation, we use an additive isotropic Gaussian noise model.

\method\ allows flexibility in the choice of the feature extraction backbone. We consider Principal Component Analysis (PCA) and Information Bottleneck (IB) \citep{tishby2000information,kolchinsky2019nonlinear} as alternatives to autoencoders for feature extraction~\citep{Bhatia2021MSTREAM}. PCA-based methods are effective for off-the-shelf learning, with little to no hyperparameter tuning. Information Bottleneck can be used for learning useful features by posing the following optimization problem:
$$
\min _{p(t | x)} I(X ; T)-\beta I(T; Y)
$$
where $X$, $Y$, and $T$ are random variables. $T$ is the compressed representation of $X$, $I(X ; T)$ and $I(T ; Y)$ are the mutual information of $X$ and $T$, and of $T$ and $Y$, respectively, and $\beta$ is a Lagrange multiplier. The problem configuration and the available data greatly influence the choice of the feature extraction algorithm. We evaluate the methods to extract features in Section \ref{exp:representation}.

\subsection{Memory}
\paragraph{\textbf{Memory-based Representation:}} The memory $\boldsymbol{M}$ is a collection of $N$ real-valued $D$ dimensional vectors where $D$ is the dimension of the encodings $\vec{z}$. Given a representation $\vec{z}$, the memory is queried to retrieve the $K$-nearest neighbours $\{{\vec{\hat{z}}_1^{t},\vec{\hat{z}}_2^{t} ... \vec{\hat{z}}_K^{t}\}}$ of $\vec{z}$ in $\boldsymbol{M}$ under the $\ell_1$ norm such that:
$$||\vec{\hat{z}}_1^{t} - \vec{z}||_1 \leq ... \leq ||\vec{\hat{z}}_K^{t} - \vec{z}||_1$$
The hyper-parameter $N$ denotes the memory size. Performance of the algorithm varies depending on the value of $N$; very large or small values of $N$ would hinder the performance.

\paragraph{\textbf{Memory Update:}} Fixed memory trained on limited samples of streaming data will not be able to handle concept drift; therefore, continuous memory update is necessary. Different memory update strategies can be used such as Least Recently Used (LRU), Random Replacement (RR), and First-In-First-Out (FIFO). We observe that the FIFO memory update policy wherein the new element to be added replaces the earliest added element in the memory works well in practice. It can easily handle concept drift in streaming data as the memory retains the most recent non-anomalous samples from the distribution. We compare FIFO with LRU and RR strategies in more detail in Section \ref{exp:memory}. It is also interesting to note that \method\ can easily handle periodic patterns by adjusting the memory size: a memory of size greater than the product of the period and the sampling frequency should be sufficient to avoid flagging periodic changes as anomalies. Section \ref{sec:drift} evaluates \method's ability to detect anomalies in a periodic setting.

As shown in Algorithm \ref{alg:memstream}, the autoencoder is initially trained with a small amount of data $\mathcal{D}$ to learn how to generate data embeddings (line 2). The memory is initialized with the same training dataset (line 3). We also store the mean and standard deviation of this small training dataset. As new records arrive, the encoder performs normalization using the stored mean and standard deviation and computes the compressed representation $\vec{z}^t$ (line 6). It then computes the $K$-nearest neighbours ($\hat{\vec{z}}^t_1, \cdots,  \hat{\vec{z}}^t_K$) by querying the memory (line 8), and calculates their $\ell_1$ distance with $\vec{z}^t$ (line 10). The final discounted score is calculated as an exponentially weighted average (weighting factor $\gamma$) (line 12). This helps in making the autoencoder more robust. The discounted score is then compared against a user-defined threshold $\beta$ (line 14) and the new record is updated into the memory in a FIFO manner if the score falls within $\beta$ (line 15). This step ensures that anomalous records do not enter the memory. If the memory is updated, then the stored mean and standard deviation are also updated accordingly. The discounted score is returned as the anomaly score for the record $\vec{x}^t$ (line 17).

\begin{algorithm}
	\caption{\method\ \label{alg:memstream}}
	\KwIn{Stream of data records}
	\KwOut{Anomaly scores for each record}
	{\bf $\triangleright$ Initialization} \\
	Feature Extractor, $f_\theta$, trained using small subset of data $\mathcal{D}$ \\
	Memory, $M$, initialized as $f_\theta(\mathcal{D})$\\
	\While{new sample $\vec{x}^t$ is received:}{
	{\bf $\triangleright$ Extract features:} \\
	$\vec{z}^t = f_\theta(\vec{x}^t)$\\
	{\bf $\triangleright$ Query memory:} \\
	$\{\vec{\hat{z}_1}^{t},\vec{\hat{z}_2}^{t} ... \vec{\hat{z}_K}^{t}\}  =$ $K$-nearest neighbours of $\vec{z}^t$ in $M$\\
	{\bf $\triangleright$ Calculate distance:} \\
	$R(\vec{z}^t,\vec{\hat{z}_i}^t) = ||\vec{z}^t-\vec{\hat{z}_i}^t||_1$ \algorithmicforall\ $i \in 1..K$\\
	{\bf $\triangleright$ Assign discounted score:} \\
	$Score(\vec{z}^t) = \dfrac{\sum_{i=1}^K{\gamma^{i-1}R(\vec{z}^t,\vec{\hat{z}_i}^t)}}{\sum_{i=1}^K{\gamma^{i-1}}}$\\
	{\bf $\triangleright$ Update Memory:} \\
	\If{$Score(\vec{z}^t) < \beta$}
	{
	    Replace earliest added element in $\boldsymbol{M}$ with $\vec{z}^t$ 
	}
	{\bf $\triangleright$ Anomaly Score:}\\
	{\bf output} $Score(\vec{z}^t)$\\
	}
\end{algorithm}

\subsection{Theoretical Analysis}

\subsubsection{Relation between Memory Size and Concept Drift}
\label{sec:theorymemsize}
Our analysis on the relation of the memory size and concept drifts suggests that the memory size should be proportional to (the spread of data distributions) / (the speed of concept drifts).

As we increase the size of memory,  we can decrease the possibility of a  false positive (falsely classifying a normal sample as an anomaly). This is because it is more likely for a new data point to have a close point in a larger memory. Therefore, on the one hand, in order to decrease the false \textit{positive} rate, we want to increase the memory size. On the other hand, in order to minimize a false \textit{negative} rate  (i.e.,  failing to raise an alarm when an anomaly did happen), Proposition \ref{prop:1} suggests that the memory size should be smaller than some quantity proportional to (standard deviations of distributions) / (the speed of distributional drifts). That is, it suggests that the memory size should be  smaller than  $2 \sigma \sqrt{d(1+\epsilon)}/\alpha$, where $d$ is the input dimension, $\alpha $ measures the speed of distributional drifts, $\sigma$ is the standard deviation of distributions, and $\epsilon \in (0,1)$.  More concretely, under drifting normal distributions, the proposition shows that a new distribution after $\tau$  drifts and an original distribution before the $\tau$ drifts are sufficiently dissimilar whenever $\tau>2 \sigma \sqrt{d(1+\epsilon)}/\alpha$, so that the memory should forget about the original distribution to minimize a false-negative rate. We also discuss this effect of increasing the memory size in Section \ref{exp:memlen}.

\begin{proposition} \label{prop:1}
(Proof in Appendix \ref{app:prop1})
Define   $S_{t,\epsilon}=\{x\in \RR^d : \|x-\mu_{t}\|_2 \le \sigma\sqrt{d(1+\epsilon)}\}$.
Let $(\mu_t)_t$ be the sequence such that there exits a positive real number $\alpha$ for which $\|\mu_t- \mu_{t'}\|_2\ge (t'-t)\alpha$ for any $t<t'$.  Let  $\tau > \frac{2 \sigma \sqrt{d(1+\epsilon)}}{\alpha}$ and    ${\displaystyle x_{t}\sim \ {\mathcal {N}}(\mu_t, \sigma I)}$ for all $t\in \NN^+$. Then, for any $\epsilon > 0$ and $t \in \NN^+ $, with probability at least $1-2 \exp(-d\epsilon^2/8)$, the following holds:  $x_t \in S_{t,\epsilon}$ and $x_{t+\tau} \notin S_{t,\epsilon}$.
\end{proposition}

\subsubsection{Architecture Choice}
\label{sec:theoryarch}
In the following, we provide one reason why we use an architecture with $d \le D$, where $d$ is the input dimension and $D$ is the embedding dimension. Namely, Proposition \eqref{prop:2} shows that if $d>D$, then there exists an anomaly constructed through perturbation of a normal sample such that the anomaly is not detectable. The construction of an anomaly in the proof is indeed unique to the case of  $d>D$, and is not applicable to the case of   $d \le D$. This provides the motivation of why we may want to use the architecture of $d \le D$, to avoid such an undetectable anomaly.

Let $\theta$ be fixed. Let  $f_\theta$ be  a  deep neural network $f_\theta: \RR^d\rightarrow \RR^D$ with ReLU and/or max-pooling as:
$f_\theta(x)=\sigma^{[L]}\big(z^{[L]}(x,\theta)\big),  z^{[l]}(x,\theta) = W^{[l]} \sigma^{(l-1)} \left(z^{[l-1]}(x,\theta)\right)$, for  $l=1,2,\dots, L$, where $\sigma^{(0)} \left(z^{[0]}(x,\theta)\right)= x$, $\sigma$ represents nonlinear function due to ReLU and/or max-pooling, and $W^{[l]}\in \RR^{N_l \times N_{l-1}}$ is a matrix of weight parameters connecting the $(l-1)$-th layer to the $l$-th layer. For the nonlinear function $\sigma$ due to ReLU and/or max-pooling, we can define $\dot \sigma^{[l]}(x,\theta)$ such that $\dot \sigma^{[l]}(x,\theta)$ is a diagonal matrix with each element being $0$ or $1$, and $\sigma^{[l]} \left(z^{[l]}(x,\theta)\right)=\dot \sigma^{[l]}(x,\theta) z^{[l]}(x,\theta)$. For any differentiable point $x$ of $f_\theta$, define $\Omega(x) = \{x'\in \RR^d:\forall l, \  \dot \sigma^{[l]}(x',\theta)= \dot\sigma^{[l]}(x,\theta)\}$ and $\Bcal_{r}(x)=\{x'\in \RR^d : \|x - x'\|_2 \le r\}$.

\begin{proposition} \label{prop:2}
(Proof in Appendix \ref{app:prop2})
Let $x$ be a differentiable point of $f_\theta$ such that   $\mathcal{B}_{r}(x) \subseteq \Omega(x)$ for some $r>0$. If $d>D$, then there exists a  $\delta \in \RR^d$ such that for any $\hx\in\RR^d$ and $\bbeta >0$,  the following holds: $\|\delta\|_{2}= r$ and 
$$
R(x,\hx) <\bbeta \implies R(x+\delta, \hx)< \bbeta.
$$
\end{proposition}

\section{Experiments}
\label{sec:exp}

In this section, we aim to answer the following questions:

\begin{enumerate}[label=\textbf{Q\arabic*.}]
\item {\bf Comparison to Streaming Methods:} How accurately does \method\ detect real-world anomalies as compared to state-of-the-art streaming baseline methods?
\item {\bf Concept Drift:} How fast can \method\ adapt under concept drift?
\item {\bf Retraining:} What effect does retraining \method\ have on the accuracy and time?
\item {\bf Self-Correction and Recovery:} Does \method\ provide a self-correction mechanism to recover from ``bad" memory states?
\end{enumerate}

\paragraph{Datasets:}
\emph{KDDCUP99} \citep{KDDCup192:online} is a popular multi-aspect anomaly detection dataset. \emph{NSL-KDD} \citep{tavallaee2009detailed} solves some of the inherent problems of \emph{KDDCUP99} such as redundant and duplicate records. Recently, \citep{ring2019survey} recommends to use \emph{UNSW-NB15} \citep{moustafa2015unsw} and \emph{CICIDS}-DoS \citep{sharafaldin2018toward} after surveying more than $30$ datasets. In addition, we use seven standard ODDS \citep{oddsdatasets} datasets: Ionosphere, Cardio, Satellite, Satimage-2, Mammograph, Pima, and ForestCover. Datasets are discussed in detail in Appendix \ref{app:datasets}.

Apart from these standard datasets, we also create and use a synthetic dataset (that we plan to release publicly), \emph{Syn} with $10\%$ anomalies and $T=10000$ samples. This dataset is constructed as a superposition of a linear wave with slope $2 \times 10^{-3}$, two sinusoidal waves with time periods $0.2T$ and $0.3T$ and amplitudes $8$ and $4$, altogether with an additive Gaussian noise from a standard normal distribution. $10\%$ of the samples are chosen at random and are perturbed with uniform random noise from the interval $[3, 6]$ to simulate anomalous data. Figure \ref{fig:syndata} shows a scatterplot of the synthetic data. Anomalous samples constitute $10\%$ of the data and are represented by red dots in the scatter plot.

    \begin{figure}[!htb]
    \begin{centering}
  \includegraphics[width=\columnwidth]{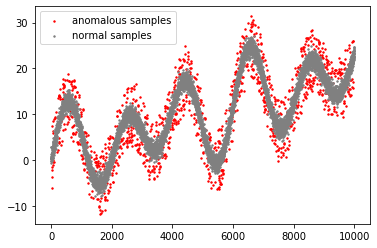}
  \caption{Scatterplot of the Synthetic Dataset.}
  \label{fig:syndata}
 \end{centering}
    \end{figure}
    
     By construction, the synthetic data distribution changes significantly over time. The presence of this concept drift makes the task challenging resulting in poor performance by baseline approaches, as seen in the Experiments. However, \method, through the use of explicit memory, can adapt to the drift in the distribution, proving its effectiveness in concept drift settings.





\begin{table*}[!htb]
\centering
\caption{AUC of \method\ and Streaming Baselines. Averaged over $5$ runs.}
\resizebox{\linewidth}{!}{
\begin{tabular}{@{}lccccccccccccc@{}}
\toprule
 \textbf{Method}
 & KDD99
 & NSL
 & UNSW
 & DoS
 & Syn.
 & Ion.
 & Cardio
 & Sat.
 & Sat.-2
 & Mamm.
 & Pima
 & Cover \\ \midrule

STORM (CIKM'07) & $0.914$ & $0.504$ & $0.810$ & $0.511$ & $0.910$ & $0.637$ & $0.507$ & $0.662$ & $0.514$ & $0.650$ & $0.528$ & $0.778$\\

{HS-Tree (IJCAI'11)} & $0.912$ & $0.845$ & $0.769$ & $0.707$ & $0.800$ & $0.764$ & $0.673$ & $0.519$ & $0.929$ & $0.832$ & $0.667$ & $0.731$ \\

{iForestASD (ICONS'13)} & $0.575$ & $0.500$ & $0.557$ & $0.529$ & $0.501$ & $0.694$ & $0.515$ & $0.504$ & $0.554$ & $0.574$ & $0.525$ & $0.603$ \\
 
{RS-Hash (ICDM'16)} & $0.859$ & $0.701$ & $0.778$ & $0.527$ & $0.921$ & $0.772$ & $0.532$ & $0.675$ & $0.685$ & $0.773$ & $0.562$ & $0.640$ \\
  
{RCF (ICML'16)} & $0.791$ & $0.745$ & $0.512$ & $0.514$ & $0.774$ & $0.675$ & $0.617$ & $0.552$ & $0.738$ & $0.755$ & $0.571$ & $0.586$ \\

{LODA (ML'16)} & $0.500$ & $0.500$ & $---$ & $0.500$ & $0.506$ & $0.503$ & $0.501$ & $0.500$ & $0.500$ & $0.500$ & $0.502$ & $0.500$\\
 
Kitsune (NDSS'18) & $0.525$ & $0.659$ & $0.794$ & $0.907$ & $---$ & $0.514$ & $0.966$ & $0.665$ & $0.973$ & $0.592$ & $0.511$ & $0.888$\\
 
{DILOF (KDD'18)} & $0.535$ & $0.821$ & $0.737$ & $0.613$ & $0.703$ & $\bf 0.928$ & $0.570$ & $0.561$ & $0.563$ & $0.733$ & $0.543$ & $0.688$ \\

{\textsc{xStream} (KDD'18)} & $0.957$ & $0.552$ & $0.804$ & $0.800$ & $0.539$ & $0.847$ & $0.918$ & $0.677$ & $\bf 0.996$ & $0.856$ & $0.663$ & $0.894$ \\
 
{\textsc{MStream} (WWW'21)} & $0.844$ & $0.544$ & $0.860$ & $0.930$ & $0.505$ & $0.670$ & $\bf 0.986$ & $0.563$ & $0.958$ & $0.567$ & $0.529$ & $0.874$ \\
 
{Ex. IF (TKDE'21)} & $0.874$ & $0.767$ & $0.541$ & $0.734$ & $---$ & $0.872$ & $0.921$ & $0.716$ & $0.995$ & $0.867$ & $0.672$ & $0.902$ \\

{\textbf{\method}} & $\bf 0.980$ & $\bf 0.978$ & $\bf 0.972$ & $\bf 0.938$ & $\bf 0.955$ & $0.821$ & $0.884$ & $\bf 0.727$ & $0.991$ & $\bf 0.894$ & $\bf 0.742$ & $\bf 0.952$ \\
 
\bottomrule
\label{tab:auc}
\end{tabular}
}
\end{table*}

\paragraph{Experimental Setup}
All methods output an anomaly score for every record (higher is more anomalous). We report the ROC-AUC (Area under the Receiver Operating Characteristic curve). All experiments, unless explicitly specified, are performed $5$ times for each parameter group, and the mean values are reported. All experiments are carried out on a $2.6 GHz$  Intel Core \textit{i}$7$ system with $16 GB$ RAM and running Mac OS Catalina $10.15.5$. Following \textsc{MStream}, we take the output dimension as $8$ for PCA and IB. For \method-PCA, we use the open-source implementation available in the scikit-learn \citep{scikit-learn} library of Principal Component Analysis. For \method-IB, we used an online implementation \footnote{\url{https://github.com/burklight/nonlinear-IB-PyTorch}} for the underlying Information Bottleneck algorithm with $\beta=0.5$ and the variance parameter set to $1$. The network was implemented as a $2$ layer binary classifier. For \method, the encoder and decoder were implemented as single layer Neural Nets with ReLU activation. We used Adam Optimizer to train both these networks with $\beta_1 = 0.9$ and $\beta_2 = 0.999$. Grid Search was used for hyperparameter tuning:  Learning Rate was set to $1\mathrm{e}-2$, and the number of epochs was set to $5000$. The memory size $N$, and the value of the threshold $\beta$, can be found in Table \ref{tab:params} in the Appendix. Memory size for each intrusion detection dataset was searched in $\{256, 512, 1024, 2048\}$. For multi-dimensional point datasets, if the size of the dataset was less than $2000$, $N$ was searched in $\{4, 8, 16, 32, 64\}$, and if it was greater than $2000$, then $N$ was searched in $\{128, 256, 512, 1024, 2048\}$. The threshold $\beta$, is an important parameter in our algorithm, and hence we adopt a finer search strategy. For each dataset, and method, $\beta$ was searched in $\{10, 1, 0.1, 0.001, 0.0001\}$. Unless stated otherwise, AE was used for feature extraction with output dimension $D=2d$, and with a FIFO memory update policy. The KNN coefficient $\gamma$  was set to $0$ for all experiments. For the synthetic dataset, we use a memory size of $N=16$. For all methods, across all datasets, the number of training samples used is equal to the memory size.

\subsection{Comparison to Streaming Methods}



Table \ref{tab:auc} shows the AUC of \method\ and state-of-the-art streaming baselines. We use open-sourced implementations of DILOF \citep{Na2018DILOFEA}, \textsc{xStream} \citep{Manzoor2018xStreamOD}, \textsc{MStream} \citep{Bhatia2021MSTREAM}, Extended Isolation Forest (Ex. IF) \citep{Hariri2021ExtendedIF}, provided by the authors, following parameter settings as suggested in the original papers. For STORM \citep{Angiulli2007DetectingDO}, HS-Tree \citep{Tan2011FastAD}, iForestASD \citep{Ding2013AnAD}, RS-Hash \citep{Sathe2016SubspaceOD}, Random Cut Forest (RCF) \citep{guha2016robust}, LODA \citep{Pevn2015LodaLO}, Kitsune \citep{Mirsky2018KitsuneAE}, we use the open-source library PySAD \citep{pysad} implementation, following original parameters. Baseline parameters are listed in Appendix \ref{app:baselines}. LODA could not process the large \emph{UNSW} dataset. Ex. IF and Kitsune are unable to run on datasets with just one field, therefore their results with \emph{Syn} are not reported.

Random subspace generation in RS-Hash includes many irrelevant features into subspaces while omitting relevant features in high-dimensional data. The objective of random projection in LODA retains the pairwise distances of the original space, therefore it fails to provide accurate outlier estimation. \textsc{xStream} performs well in \emph{KDD99}, \textsc{MStream} performs well in \emph{DoS}, however, note that \method\ achieves statistically significant improvements in AUC scores over baseline methods. Moreover, baselines are unable to catch complicated drift scenarios in \emph{NSL}, \emph{UNSW} and \emph{Syn}.

\begin{table}[!htb]
\centering
\caption{AUC-PR and Time required to run \method\ and Streaming Baselines on \emph{NSL-KDD}. \method\ provides statistically significant (p value $< 0.001$) improvements over baseline methods.}
\begin{tabular}{@{}lcc@{}}
\toprule
 \textbf{Method}
 & \textbf{AUC-PR}
 & \textbf{Time (s)}
 \\ \midrule
STORM & $0.681 \pm 0.000$ & $754$\\
{HS-Tree} & $0.709 \pm 0.063$ & $306$\\
{iForestASD} & $0.534 \pm 0.000$ & $19876$\\
{RS-Hash} & $0.500 \pm 0.140$  & $892$\\
{RCF} & $0.664 \pm 0.006$ & $665$\\
{LODA} & $0.734 \pm 0.067$ & $2617$\\
Kitsune & $0.673 \pm 0.000$ & $821$\\
{DILOF} & $0.822 \pm 0.000$ & $260$\\
{\textsc{xStream}} & $0.541 \pm 0.070$ & $34$\\
{\textsc{MStream}} & $0.510 \pm 0.000$ & $0.08$\\
{Ex. IF} & $0.659 \pm 0.014$ & $889$\\
\textbf{\method} & $\mathbf{0.959} \pm 0.002$ & $55$\\
\bottomrule
\label{tab:time}
\end{tabular}
\end{table}

Table \ref{tab:time} reports the running AUC-PR scores of \method\ and baseline methods on the \emph{NSL-KDD} dataset, as well as their corresponding running times. Note that not only does \method\ greatly outperform baselines on AUC-PR, but also does so in a time-efficient manner.




\subsection{Concept Drift}
\label{sec:drift}

We next investigate \method's performance under concept drift, particularly how fast it can adapt. As shown in Figure \ref{fig:drift} (top), we create a synthetic data set which covers a wide variety of drifts scenarios: (a) point anomalies: $T=19000$ (b) sudden frequency change: $T\in[5000, 10000]$ (c) continuous concept drift: $T\in [15000, 17500]$ (d) sudden concept drift due to mean change: $T\in [12500, 15000]$. Anomaly scores are clipped at $T=12500$ and $T=19000$ for better visibility.

\method\ is able to handle all the above-mentioned concept drift scenarios as is evident in Figure \ref{fig:drift} (bottom). We observe that \method\ assigns high scores corresponding to trend-changing events (e.g. $T=1000, 5000, 10000$ etc.) which produce anomalies, then with a gradual decrease in scores thereafter as it \emph{adapts} successfully to the new distribution. Note that \method\ can also adapt to periodic streams. For the first cycle of the sine wave $T\in[1000, 2000]$, the anomalous scores are relatively high. However, as more and more normal samples are seen from the sine distribution, \method\ adapts to it.


\begin{figure}[!htb]
\centering
  \includegraphics[width=\columnwidth]{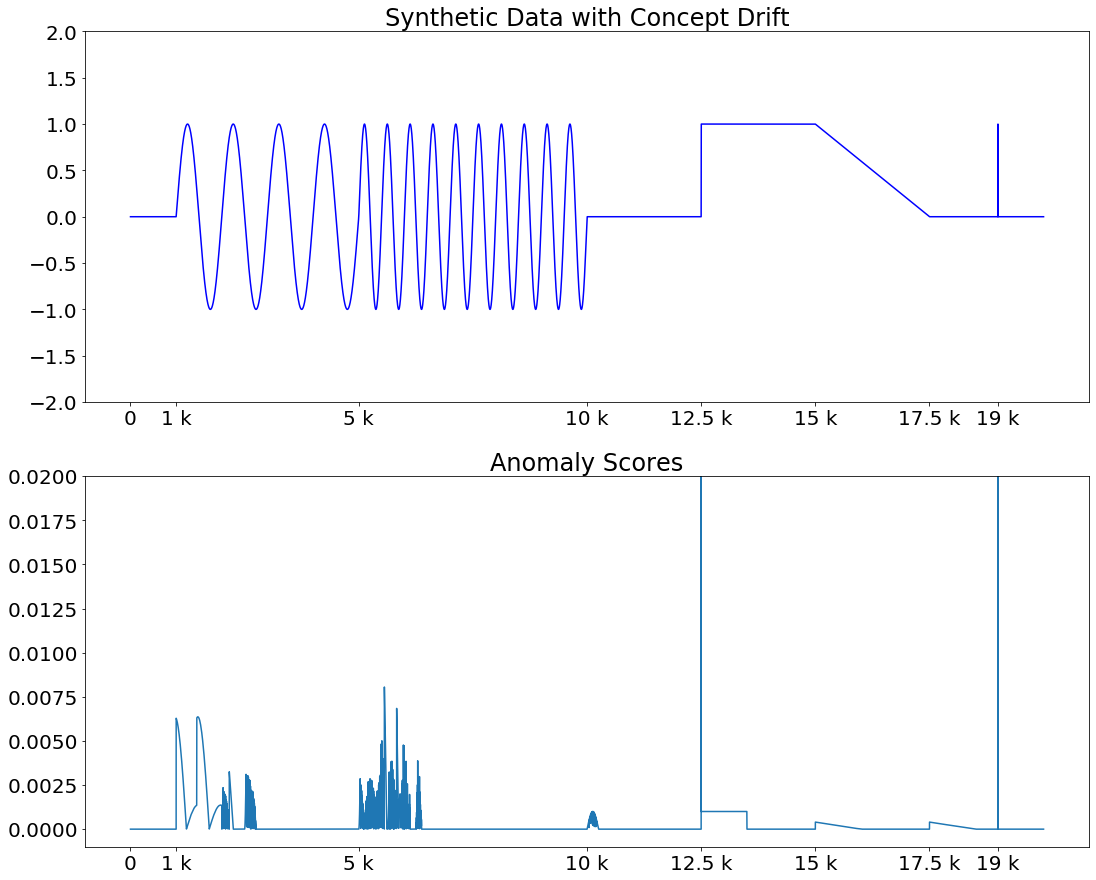}
  \captionof{figure}{(Top): Synthetic data with drift. (Bottom): Anomaly Scores output by \method\ demonstrating resilience to drift.}
  \label{fig:drift}
\end{figure}

\subsection{Retraining}
\label{sec:retrain}
 The need for re-training is especially prevalent in very long drifting streams where the feature extractor, trained on the small subset of the initial normal data $\mathcal{D}$, starts facing record data sufficiently different from its training data. In this experiment, we test the ability of \method\ to accommodate this more challenging setting by periodically retraining its feature extractor. Fine-tuning is performed at regular intervals distributed uniformly across the stream, i.e. to implement $k$ fine-tunings on a stream of size $S$, the first fine-tuning occurs at $\left \lfloor \frac{S}{k+1} \right \rfloor$. Figure \ref{fig:retrain} shows the AUC and time taken to fine-tune \method\ on \emph{CICIDS-DoS} with a stream size greater than $1M$ records. Note that as we increase the number of times \method\ is fine-tuned, we observe large gains in AUC with the negligible time difference.

\begin{table*}[!htb]
\centering
\caption{Effect of Memory Size on the AUC in \method\ on \emph{NSL-KDD} dataset.}
\label{tab:memoryeffect}
\begin{tabular}{@{}rcccccccccccc@{}}
\toprule
\textbf{Memory Size} & $2^4$ & $2^5$ & $2^6$ & $2^7$ & $2^8$ & $2^9$ & $2^{10}$ & $2^{11}$ & $2^{12}$ & $2^{13}$ & $2^{14}$ \\

\midrule
\textbf{AUC} & $0.670$ & $0.649$ & $0.932$ & $0.936$ & $0.923$ & $0.950$ & $0.972$ & $0.976$ & $0.985$ & $0.989$ & $0.991$ \\
\bottomrule
\end{tabular}
\end{table*}

\begin{figure}[!htb]
\begin{tikzpicture}
\pgfplotsset{
      scale only axis,
      width=0.33\textwidth,height=0.3\textwidth
  }

  \begin{axis}[
    axis y line*=left,
    xlabel=Number of times Fine-Tuned,
    xtick = {0, 1, 2, 3, 4, 5, 6, 7, 8, 9},
    ymin=0.8,
    ymax=1.0,
    tick label style={font=\normalsize},
    ylabel = {AUC},
  ]
    \addplot[red!70!black,mark=*] coordinates {(0, 0.8302751955)
(1, 0.8974884823)
(2, 0.9170457649)
(3, 0.9328035165)
(4, 0.9317918228)
(5, 0.9468332062)
(6, 0.9492323078)
(7, 0.9477016701)
(8, 0.947640945)
(9, 0.9541160801)
    };\label{auc}
    \end{axis}

    \begin{axis}[
      axis y line*=right,
      axis x line=none,
      xtick = {0, 1, 2, 3, 4, 5, 6, 7, 8, 9},
      ymode = log,
      log ticks with fixed point,
        ytick = {10, 100, 500, 1000},
        ymin=10,
        ymax=1000,
        yticklabels = {$10$, $100$, $500$, $1000$},
        tick label style={font=\normalsize},
        ylabel = {Time (in s)},
        legend style={at={(0.5,0.1)},legend columns=2,fill=none,draw=black,anchor=center,align=center},
    ]

    \addlegendimage{/pgfplots/refstyle=auc}\addlegendentry{AUC}
    \addplot[brown,mark=star] coordinates {(0, 481.9602451)
(1, 518.6869633)
(2, 511.7169969)
(3, 486.0186646)
(4, 533.2838132)
(5, 501.0282154)
(6, 504.4203665)
(7, 530.9628208)
(8, 562.8917241)
(9, 546.230355)
    };     
    \addlegendentry{Time}; 
  \end{axis}


\end{tikzpicture}
\caption{Retraining effect on the AUC and time for \emph{CICIDS-DOS}.}
\label{fig:retrain}
\end{figure}
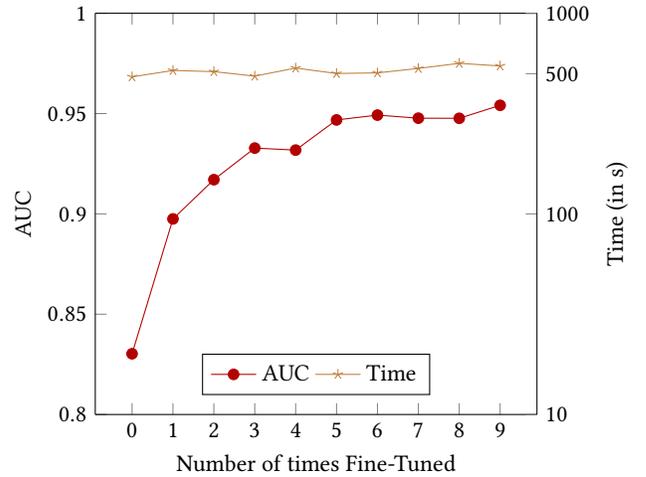


\subsection{Self-Correction and Recovery}
\label{sec:recovery}

Consider the scenario where an anomalous element enters into the memory. A particularly catastrophic outcome of this event could be the cascading effect where more and more anomalous samples replace the normal elements in the memory due to their similarity. This can ultimately lead to a situation where the memory solely consists of anomalous samples. These ``Group Anomaly" events are fairly common in intrusion detection settings. We show that this issue is mitigated by the use of $K$-nearest neighbours in our approach. We simulate the above setting by adding the first labeled anomalous element in memory during the initialization.

\begin{table}[!htb]
\centering
\caption{Performance of \method\ on \emph{NSL-KDD} dataset after adding an anomalous element in memory when $K=3$ and for different values of discount factor $\gamma$.}
\label{tab:recovery}
\begin{tabular}{lrr}
    \toprule
    \textbf{$\gamma$} & High $\beta (=1)$ & Appropriate $\beta (=0.001)$ \\
    \midrule
    $0$ & $0.771$ & $0.933$\\ 
    $0.25$ & $0.828$ & $0.966$\\ 
    $0.5$ & $0.848$ & $0.967$\\ 
    $1$ & $0.888$ & $0.965$\\
    \bottomrule
  \end{tabular}
\end{table}

In Table \ref{tab:recovery}, a high $\beta$ allows anomalous elements to also enter the memory. In the absence of $K$-nearest neighbour discounting (i.e. $\gamma=0$), a high $\beta$ value algorithm succumbs to the above-described scenario resulting in poor performance. On the other hand, with discounting (i.e. $\gamma \neq 0$), the algorithm is able to ``recover" itself, and as a result, the performance does not suffer considerably. Note that when the threshold $\beta$ is in its appropriate range, the algorithm is robust to the choice of discount factor $\gamma$.


\subsection{Ablations}
\label{sec:ablations}

\begin{table}[!htb]
    \centering
    \caption{Ablation study for different components of \method\ on \emph{KDDCUP99}.}
    \label{tab:ablations}
    \begin{tabular}{ll|rrrr}
    \toprule
         & \textbf{Component} & \multicolumn{4}{c}{\bf Ablations}\\
         \midrule
         (a) & Memory & None & LRU & RR & FIFO\\
         & Update & $0.938$& $0.946$& $0.946$& $0.980$\\
         \midrule
         (b) & Feature & Identity &PCA & IB & AE\\
         & Extraction & $0.822$& $0.863$ & $0.959$& $0.980$\\
         \midrule
         (c) & Memory & 128 & 256 & 512 & 1024\\
         & Length ($N$) & $0.950$& $0.980$& $0.946$& $0.811$\\
         \midrule
         (d) & Output & $d/2$ & $d$ & $2d$ & $5d$\\
         & Dimension ($D$) & $0.951$& $0.928$& $0.980$& $0.983$\\
         \midrule
         (e) & Update & 1 & 0.1 & 0.01 & 0.001\\
         & Threshold ($\beta$) & $0.980$& $0.938$& $0.938$& $0.938$ \\
         \midrule
         (f) & KNN & 0 & 0.25 & 0.5 & 1\\
         & coefficient ($\gamma$) & $0.980$& $0.939$& $0.937$ & $0.936$ \\
    \bottomrule
    \end{tabular}
\end{table}

\textbf{(a) Memory Update:}
\label{exp:memory}
Taking inspiration from the work done in cache replacement policies in computer architecture, we replace the FIFO memory update policy with Least Recently Used (LRU) and Random Replacement (RR) policies. Table \ref{tab:ablations}(a) reports results with these three and when no memory update is performed on the \emph{KDDCUP99} dataset. Note that FIFO outperforms other policies. This is due to the temporal locality preserving property of the FIFO policy to keep track of the current trend. LRU and RR policies do not maintain a true snapshot of the stream in the memory and are thus unable to learn the changing trend.

\textbf{(b) Feature Extraction:}
\label{exp:representation}
Table \ref{tab:ablations}(b) shows experiments with different methods for feature extraction discussed in Section \ref{sec:representation}. Autoencoder outperforms both PCA and Information Bottleneck approaches.

\textbf{(c) Memory Length ($N$):}
\label{exp:memlen} 
As we noted in Section \ref{sec:theorymemsize}, increasing $N$ can decrease the false positive rate, but also increase the false negative rate. We observe this effect empirically in Table \ref{tab:ablations}(c), where the sweet spot is found at $N=256$, and increasing memory length further degrades performance.  An additional experiment demonstrating the effect of memory size is discussed in Table \ref{tab:memoryeffect}. We note that very large or very small values of N would hinder the algorithm performance as the memory will not be able to capture the current trend properly. A very large `N' will not ensure that the current trend is learned exclusively and the memory would always be contaminated by representatives of the previous trend. On the other hand, a very small `N' will not allow enough representatives from the current trend and thus in both cases, the performance of the algorithm will be sub-optimal.

\textbf{(d) Output Dimension ($D$):}
\label{exp:output}
In Section \ref{sec:theoryarch}, we motivate why we use an architecture with $D >= d$. In Table \ref{tab:ablations}(d), we compare architectures with different output dimension $D$ as a function of the input dimension $d$. We find that $D=d/2$ outperforms an architecture with $D=d$, owing to the features learning by dimensionality reduction. Note that \method\ performs well for large $D$.

\textbf{(e) Update Threshold ($\beta$):}
\label{exp:threshold}
The update threshold is used to judge records based on their anomaly scores and determine whether they should update the memory. A high $\beta$ corresponds to frequent updates to the memory, whereas a low $\beta$ seldom allows memory updates. Thus, $\beta$ can capture our belief about how frequently the memory should be updated, or how close is the stream to the initial data distribution. From Table \ref{tab:ablations}(e), we notice that for \emph{KDDCUP99}, a drifting dataset, a more flexible threshold ($\beta=1$) performs well, and more stringent thresholds perform similar to no memory updates (Table \ref{tab:ablations}(a)).

\textbf{(f) KNN coefficient ($\gamma$):}
\label{exp:knncoeff}
In Section \ref{sec:recovery}, we discussed the importance of the KNN coefficient $\gamma$ in the Self-Recovery Mechanism. Table \ref{tab:ablations}(f) compares different settings of $\gamma$, without memory poisoning.

\section{Conclusion}
We propose \method, a novel memory augmented feature extractor framework for streaming anomaly detection in multi-dimensional data and concept drift settings. \method\ uses a denoising autoencoder to extract features and a memory module with a FIFO replacement policy to learn the dynamically changing trends. Moreover, \method\ allows quick retraining when the arriving stream becomes sufficiently different from the training data. We give a theoretical guarantee on the relation between the memory size and the concept drift. Furthermore, \method\ prevents memory poisoning by using (1) a discounting $K$-nearest neighbour memory leading to a unique self-correcting and recovering mechanism; (2) a theoretically motivated architecture design choice. \method\ outperforms $11$ state-of-the-art streaming methods. Future work could consider more tailored memory replacement policies, e.g. by assigning different weights to the memory elements.

\FloatBarrier

\bibliographystyle{unsrt}
\bibliography{references}


\begin{table*}[!htb]
\centering
\caption{Statistics of the datasets.}
\begin{tabular}{@{}lccccccccccccc@{}}
\toprule
 
 & KDD99
 & NSL
 & UNSW
 & DoS
 & Syn.
 & Ion.
 & Cardio
 & Sat.
 & Sat.-2
 & Mamm.
 & Pima
 & Cover \\ \midrule

\textbf{Records} & $494,021$ & $125,973$ & $2,540,044$ & $1,048,575$ & $10,000$ & $351$ & $1831$ & $6435$ & $5803$ & $11183$ & $768$ & $286048$ \\

\textbf{Dimensions} & $121$ & $126$ & $122$ & $95$ & $1$ & $33$ & $21$ & $36$ & $36$ & $6$ & $8$ & $10$ \\

\bottomrule
\label{tab:datasets}
\end{tabular}
\end{table*}

\begin{table*}[!htb]
\centering
\caption{Memory Length and Update Threshold used for the different datasets}
\begin{tabular}{@{}lccccccccccccc@{}}
\toprule
 \textbf{Method}
 & KDD99
 & NSL
 & UNSW
 & DoS
 & Syn.
 & Ion.
 & Cardio
 & Sat.
 & Sat.-2
 & Mamm.
 & Pima
 & Cover \\ \midrule

$N$ & $256$ & $2048$ & $2048$ & $2048$ & $16$ & $4$ & $64$ & $32$ & $256$ & $128$ & $64$ & $2048$\\

$\beta$ & $1$ & $0.1$ & $0.1$ & $0.1$ & $1$ & $0.001$ & $1$ & $0.01$ & $10$ & $0.1$ & $0.001$ & $0.0001$  \\

\bottomrule
\label{tab:params}
\end{tabular}
\end{table*}

\appendix
\section*{Appendix}

\setcounter{proposition}{0}

\section{Proofs}

\begin{proposition} \label{app:prop1}
Define   $S_{t,\epsilon}=\{x\in \RR^d : \|x-\mu_{t}\|_2 \le \sigma\sqrt{d(1+\epsilon)}\}$.
Let $(\mu_t)_t$ be the sequence such that there exits a positive real number $\alpha$ for which $\|\mu_t- \mu_{t'}\|_2\ge (t'-t)\alpha$ for any $t<t'$.  Let  $\tau > \frac{2 \sigma \sqrt{d(1+\epsilon)}}{\alpha}$ and    ${\displaystyle x_{t}\sim \ {\mathcal {N}}(\mu_t, \sigma I)}$ for all $t\in \NN^+$. Then, for any $\epsilon > 0$ and $t \in \NN^+ $, with probability at least $1-2 \exp(-d\epsilon^2/8)$, the following holds:  $x_t \in S_{t,\epsilon}$ and $x_{t+\tau} \notin S_{t,\epsilon}$.
\end{proposition}

\begin{proof}
Let us write $\bd(x,x')=\|x-x'\|_2$. Then, by the triangle inequality,
\begin{align} \label{eq:1}
\bd(\mu_t, \mu_{t+\tau}) \le \bd(\mu_t, x_{t+\tau})+\bd( x_{t+\tau}, \mu_{t+\tau}). 
\end{align} 
By using the property of  the Gaussian distribution with   ${\displaystyle z_{t+\tau}\sim \ {\mathcal {N}}(0,  I)}$, we have that 
\begin{align*}
& \Pr(\| x_{t+\tau}-\mu_{t+\tau}\|_{2}< \sigma\sqrt{d(1+\epsilon)} ) \\
& = \Pr(\| \sigma z_{t+\tau}+\mu_{t+\tau}-\mu_{t+\tau}\|_{2}< \sigma\sqrt{d(1+\epsilon)} ) \\
& =\Pr(\| z_{t+\tau}\|_{2}^{2}< d(1+\epsilon) ). 
\end{align*}

Thus, using the Chernoff bound for the Standard normal distribution  for  ${\displaystyle z_{t+\tau}\sim \ {\mathcal {N}}(0,  I)}$, we have that
\begin{align*}
\Pr(\| x_{t+\tau}-\mu_{t+\tau}\|_{2}>\sigma\sqrt{d(1+\epsilon)} ) \le \exp\left(-\frac{d\epsilon^2}{8}\right).
\end{align*}
Similarly, 
\begin{align*}
\Pr(\| x_{t}-\mu_{t}\|_{2}>\sigma\sqrt{d(1+\epsilon)} ) \le \exp\left(-\frac{d\epsilon^2}{8}\right).
\end{align*}
By tanking union hounds, we have  that with probability at least $1-2 \exp(-d\epsilon^2/8)$,
\begin{align} \label{eq:2}
\| x_{t+\tau}-\mu_{t+\tau}\|_{2}\le\sigma\sqrt{d(1+\epsilon)},
\end{align}
and
\begin{align} \label{eq:3}
\| x_{t}-\mu_{t}\|_{2}\le\sigma\sqrt{d(1+\epsilon)}.
\end{align}
By using the upper bound of \eqref{eq:2} in \eqref{eq:1}, we have that $\bd(\mu_t, \mu_{t+\tau}) \le \bd(\mu_t, x_{t+\tau})+\sigma\sqrt{d(1+\epsilon)}$,
which implies that  
\begin{align*} 
\bd(\mu_t, \mu_{t+\tau}) -\sigma\sqrt{d(1+\epsilon)}\le \bd(\mu_t, x_{t+\tau}).
\end{align*}
Using the assumption on  $(\mu_t)_t$, 
\begin{align*} 
\tau\alpha-\sigma\sqrt{d(1+\epsilon)}\le \bd(\mu_t, x_{t+\tau}).
\end{align*}
Using the definition of $\tau$, 
\begin{align*} 
 \sigma \sqrt{d(1+\epsilon)}< \bd(\mu_t, x_{t+\tau}).
\end{align*}
This means that $x_{t+\tau} \notin S_{t,\epsilon}$. On the other hand, equation \eqref{eq:3} shows that $ x_{t} \in S_{t,\epsilon}$. 
\end{proof}

\begin{proposition} \label{app:prop2}
Let $x$ be a differentiable point of $f_\theta$ such that   $\mathcal{B}_{r}(x) \subseteq \Omega(x)$ for some $r>0$. If $d>D$, then there exists a  $\delta \in \RR^d$ such that for any $\hx\in\RR^d$ and $\bbeta >0$,  the following holds: $\|\delta\|_{2}= r$ and 
$$
R(x,\hx) <\bbeta \implies R(x+\delta, \hx)< \bbeta.
$$
\end{proposition}

\begin{proof}
 We can rewrite the output of the function as
$
f_\theta(x)=\dot \sigma^{[L]}(x,\theta) W^{[L]} \dot \sigma^{[L-1]}(x,\theta)  W^{[L-1]}  \cdots 
W^{[2]}\dot \sigma^{[1]}(x,\theta) W^{[1]} x.
$
Thus, for any $\delta$ such that $(x+\delta) \in \mathcal{B}_{r}(x) \subseteq \Omega(x)$, we have
\begin{align*}
f_\theta(x+\delta)& =\dot \sigma^{[L]}(x+\delta,\theta) W^{[L]} \dot \sigma^{[L-1]}(x+\delta,\theta)  W^{[L-1]}  \cdots \\
& W^{[2]}\dot \sigma^{[1]}(x+\delta,\theta) W^{[1]} (x+\delta)
\\
& =\sigma^{[L]}(x,\theta) W^{[L]} \dot \sigma^{[L-1]}(x,\theta)  W^{[L-1]}  \cdots \\
& W^{[2]}\dot \sigma^{[1]}(x,\theta) W^{[1]} (x+\delta)
 \\ & = M x + M \delta 
\end{align*}
where $M = \sigma^{[L]}(x,\theta) W^{[L]} \dot \sigma^{[L-1]}(x,\theta)  W^{[L-1]}  \\ \cdots W^{[2]}\dot \sigma^{[1]}(x,\theta) W^{[1]}$. Notice that $M$ is a matrix of size $D$ by $d$. Thus,  ff $d>D$, there the nulls space (or the kernel space) of $M$ is not $\{0\}$ and there exists $\delta'\in\RR^d $ in the null space of $M$ such that $\|\delta'\|\neq 0$ and $M (r'\delta')=0$ for all $r'>0$. Thus, there exists a  $\delta \in \RR^d$ such that  $(x+\delta) \in \mathcal{B}_{r}(x) \subseteq \Omega(x)$,   $\|\delta\|_2 = r$, and $M \delta=0$, yielding
$$
f_\theta(x+\delta)=M x=f_\theta(x). 
$$
This implies the statement of this proposition. 
\end{proof}

\section{Datasets}
\label{app:datasets}

Table \ref{tab:datasets} contains the datasets that we use for evaluation. We briefly describe how these datasets are prepared for anomaly detection.

\begin{enumerate}
    \item \emph{KDDCUP99} \citep{KDDCup192:online} is based on the DARPA data set and is amongst the most extensively used data sets for multi-aspect anomaly detection. The original dataset contains samples of $41$ dimensions, $34$ of which are continuous and $7$ are categorical, and also displays concept drift \citep{minku2011ddd}. We use one-hot representation to encode the categorical features, and eventually, we obtain a dataset of $121$ dimensions. For the \emph{KDDCUP99} dataset, we follow the settings in \citep{zong2018deep}. As $20\%$ of data samples are labeled as ``normal" and $80\%$ are labeled as ``attack", normal samples are in a minority group; therefore, we treat normal ones as anomalous in this experiment and the $80\%$ samples labeled as attack in the original dataset are treated as normal samples.

    \item \emph{NSL-KDD} \citep{tavallaee2009detailed} solves some of the inherent problems of the \emph{KDDCUP99} dataset such as redundant and duplicate records and is considered more enhanced as compared to \emph{KDDCUP99}.

    \item \emph{CICIDS-DoS} \citep{sharafaldin2018toward} was created by the Canadian Institute of Cybersecurity. Each record is a flow containing features such as Source IP Address, Source Port, Destination IP Address, Bytes, Packets. These flows were captured from a real-time simulation of normal network traffic and synthetic attack simulators. This consists of the \emph{CICIDS-DoS} dataset ($1.05$ million records). \emph{CICIDS-DoS} has $5\%$ anomalies and contains samples of $95$ dimensions with a mixture of numeric and categorical features. For categorical features, we further used binary encoding to represent them because of the high cardinality.

    \item \emph{UNSW-NB15} \citep{moustafa2015unsw} was created by the Cyber Range Lab of the Australian Centre for Cyber Security (ACCS) for generating a hybrid of real modern normal activities and synthetic contemporary attack behaviors. This dataset has nine types of attacks, namely, Fuzzers, Analysis, Backdoors, DoS, Exploits, Generic, Reconnaissance, Shellcode, and Worms. It has $13\%$ anomalies.


    

    \item Ionosphere \citep{oddsdatasets} is derived using the ionosphere dataset from the UCI ML repository \citep{ucidatasets} which is a binary classification dataset with dimensionality $34$. There is one attribute having values of all zeros, which is discarded. So the total number of dimensions is $33$. The `bad' class is considered as outliers class and the `good' class as inliers.
    
    \item Cardio \citep{oddsdatasets} is derived using the Cardiotocography (Cardio) dataset from the UCI ML repository \citep{ucidatasets} which consists of measurements of fetal heart rate (FHR) and uterine contraction (UC) features on cardiotocograms classified by expert obstetricians. This is a classification dataset, where the classes are normal, suspect, and pathologic. For outlier detection, the normal class formed the inliers, while the pathologic (outlier) class is downsampled to $176$ points. The suspect class is discarded.
    
    \item Satellite \citep{oddsdatasets} is derived using the Statlog (Landsat Satellite) dataset from the UCI ML repository \citep{ucidatasets} which is a multi-class classification dataset. Here, the training and test data are combined. The smallest three classes, i.e. $2, 4, 5$ are combined to form the outliers class, while all the other classes are combined to form an inlier class.
    
    \item Satimage-2 \citep{oddsdatasets} is derived using the Statlog (Landsat Satellite) dataset from the UCI ML repository \citep{ucidatasets} which is also a multi-class classification dataset. Here, the training and test data are combined. Class $2$ is down-sampled to $71$ outliers, while all the other classes are combined to form an inlier class. The modified dataset is referred to as Satimage-$2$.
    
    \item Mammography \citep{oddsdatasets} is derived from openML\footnote{https://www.openml.org/}. The publicly available openML dataset has $11,183$ samples with $260$ calcifications. If we look at predictive accuracy as a measure of goodness of the classifier for this case, the default accuracy would be $97.68\%$ when every sample is labeled non-calcification. But, it is desirable for the classifier to predict most of the calcifications correctly. For outlier detection, the minority class of calcification is considered as the outlier class and the non-calcification class as inliers.
    
    \item Pima \citep{oddsdatasets} is the same as Pima Indians diabetes dataset of the UCI ML repository \citep{ucidatasets} which is a binary classification dataset. Several constraints were placed on the selection of instances from a larger database. In particular, all patients here are females at least $21$ years old of Pima Indian heritage.
    
    \item ForestCover \citep{oddsdatasets} is the ForestCover/Covertype dataset from the UCI ML repository \citep{ucidatasets} which is a multiclass classification dataset. It is used in predicting forest cover type from cartographic variables only (no remotely sensed data). This study area includes four wilderness areas located in the Roosevelt National Forest of northern Colorado. This dataset has $54$ attributes ($10$ quantitative variables, $4$ binary wilderness areas, and $40$ binary soil type variables). Here, an outlier detection dataset is created using only $10$ quantitative attributes. Instances from class $2$ are considered as normal points and instances from class $4$ are anomalies. The anomalies ratio is $0.9\%$. Instances from the other classes are omitted.

\end{enumerate}

\section{Memory Size and Update Thresholds}
Table \ref{tab:params} shows the memory size $N$, and the value of the threshold $\beta$.

\section{Baseline Parameters}
\label{app:baselines}

\paragraph{STORM:} window\_size=$10000$, max\_radius=$0.1$
\paragraph{HS-Tree:} window\_size=$100$, num\_trees=$25$, max\_depth=$15$, initial\_window\_X=None
\paragraph{iForestASD:}	window\_size=$100$, n\_estimators=$25$, anomaly\_threshold=$0.5$, drift\_threshold=$0.5$
\paragraph{RS-Hash:} sampling\_points=$1000$, decay=$0.015$, num\_components=$100$, num\_hash\_fns=$1$
\paragraph{RCF:} num\_trees=$4$, shingle\_size=$4$, tree\_size=$256$
\paragraph{LODA:} num\_bins=$10$, num\_random\_cuts=$100$
\paragraph{Kitsune:} max\_size\_ae=$10$, learning\_rate=$0.1$, hidden\_ratio=$0.75$, grace\_feature\_mapping=grace\_anomaly\_detector=$10\%$ of data,
\paragraph{DILOF:} window size = $400$, thresholds = [0.1f, 1.0f, 1.1f, 1.15f, 1.2f, 1.3f, 1.4f, 1.6f, 2.0f, 3.0f] , K = $8$
\paragraph{\textsc{xStream}:} projection size=$50$, number of chains=$50$, depth=$10$, rowstream=$0$, nwindows=$0$, initial sample size=\# rows in data, scoring batch size=$100000$
\paragraph{\textsc{MStream}:} alpha = $0.85$
\paragraph{Ex. IF:} ntrees=$200$, sample\_size=$256$, limit=None, ExtensionLevel=$1$

\end{document}